\pdfoutput=1
\documentclass{article}

\usepackage{colt11e}
\usepackage[round,comma]{natbib}
\usepackage{algorithm,algorithmic}
\usepackage{amsmath, amssymb, multirow, paralist}
\newtheorem{thm}{Theorem}

\def \R {\mathbb{R}}

\def \E {\mathrm{E}}
\def \a {\mathbf{a}}

\def \v {\mathbf{v}}

\def \X {\mathcal{X}}

\def \Mh {\widehat{M}}

\def \u {\mathbf{u}}

\def \v {\mathbf{v}}

\def \R {\mathbb{R}}

\def \b {\mathbf{b}}

\def \vt {\widetilde{\v}}

\def \Hh {\widehat{H}}

\def \e {\mathbf{e}}

\def \Uh {\widehat{U}}
\def \uh {\widehat{\u}}
\def \Vh {\widehat{V}}
\def \vh {\widehat{\v}}

\def \vec {\mbox{vec}}
\def \k {\mathbf{k}}

\def \Rt {\mathcal{R}}

\def \ut {\widetilde{\u}}
\def \muh {\widehat{\mu}}

\title{CUR Algorithm with Incomplete Matrix Observation}
\author{Rong Jin$^{\dagger}$ \and Shenghuo Zhu$^{\ddagger}$ \\
$^{\dagger}$ Dept. of Computer Science and Engineering, Michigan State University, \texttt{\small rongjin@msu.edu} \\
$^{\ddagger}$  NEC Laboratories America, Inc., \texttt{\small
  zsh@nec-labs.com}
}
\begin{document}
\maketitle

\section{Introduction}

CUR matrix decomposition is a randomized algorithm that can efficiently compute the low rank approximation for a given rectangle matrix~\citep{drineas-2006-fast, drineas-2008-relative,mahoney-2009-cur}. Let $A\in \R^{n\times m}$ be the given matrix and $k$ be the target rank for approximation. CUR randomly samples $c = O(k\log k/\epsilon^2)$ columns and $r= O(k\log k/\epsilon^2)$ rows from $A$, according to their leverage scores, to form matrices $C$ and $R$, respectively. The approximated matrix $\widehat{A}$ is then computed as $CUR$, where $U = C^{\dagger} A R^{\dagger}$. It can be shown, that with a high probability,
\begin{eqnarray}
\|A - \widehat{A}\|_F \leq (2 + \epsilon) \|A - A_k\|_F \label{eqn:bound-1}
\end{eqnarray}
where $A_k$ is the best $k$-rank approximation of $A$. In case when the maximum of statistical leverage scores, which is also referred to as incoherence measure in matrix completion~\citep{candes-2010-power,recht-2011-simple,candes-2012-exact}, are small, CUR matrix decomposition can be simplified by uniformly sampling rows an columns from $A$. The simplified algorithm will have a relative error bound similar to that in (\ref{eqn:bound-1}) except that the sample sizes $c$ and $r$ should be increased by the incoherence measure. In this draft, we will focus on the situation with bounded incoherence measure where uniform sampling of columns and rows is in general sufficient.

One limitation with the existing CUR algorithms is that they require
an access to the full matrix $A$ for computing $U$. In this work, we
aim to alleviate this limitation. In particular, we assume that
besides having an access to randomly sampled $d$ rows and $d$ columns
from $A$, we only observe a subset of randomly sampled entries
$\Omega$ from $A$. Our goal is to develop a low rank approximation
algorithm, similar to CUR, based on (i) randomly sampled rows and
columns from $A$, and (ii) randomly sampled entries from $A$.

Compared to the standard matrix completion theory~\citep{candes-2010-power,recht-2011-simple,candes-2012-exact}, the key advantage of the proposed algorithm is its low sample complexity and high computational efficiency. In particular, unlike matrix completion that requires $O(rn\log^2n)$ number of observed entries, the proposed algorithm is able to perfectly recover the target matrix $A$ with only $O(rn\log n)$ number of observed entries (including the randomly sampled entries and entries in randomly sampled rows and columns). In addition, instead of having to solve an optimization problem involved trace norm regularization, the proposed algorithm only needs to solve a standard regression problem. Finally, unlike most matrix completion theories that hold only when the target matrix is of low rank, we show a strong guarantee for the proposed algorithm even when the target matrix $A$ is not low rank.

We finally note that a closely related algorithm, titled ``Low-rank Matrix and Tensor Completion via Adaptive Sampling'', was published recently~\citep{krishnamurthy-20130-low}. It is designed to recover a low rank matrix with randomly sampled rows and entries, which is different from the goal of this work (i.e. computing a low rank approximation for a target matrix $A$).

\section{Algorithm and Notation}
Let $M \in \R^{n\times m}$ be the target matrix, where $n \geq m$. To
approximate $M$, we first sample uniformly at random $d$ columns and
rows from $M$, denoted by $A = (\a_1, \ldots, \a_d) \in \R^{n\times
  d}$ and $B = (\b_1, \ldots, \b_d) \in \R^{m\times d}$, respectively,
where each $\a_i \in \R^n$ and $\b_j \in \R^m$ is a row and column of
$M$, respectively. Let $r$ be the target rank for approximation, with
$r \leq d$. Let $\Uh = (\uh_1, \ldots, \uh_r) \in \R^{n\times r}$ and
$\Vh = (\vh_1, \ldots, \vh_r) \in \R^{m\times r}$ be the first $r$
eigenvectors of $AA^{\top}$ and $BB^{\top}$, respectively. Besides $A$
and $B$, we furthermore sample, uniformly at random, entries from
matrix $M$. Let $\Omega$ include the indices of randomly sampled
entries. Our goal is to approximately recover the matrix $M$ using
$A$, $B$, and randomly sample entries in $\Omega$. To this end, we
will solve the following optimization problem
\begin{eqnarray}
\min\limits_{Z \in \R^{r\times r}} |\Rt_{\Omega}(M) - \Rt_{\Omega}(\Uh Z \Vh^{\top})|_F^2 \label{eqn:opt}
\end{eqnarray}
where $\Rt_{\Omega}:\R^{n\times m} \mapsto \R^{n\times m}$ is defined as
\[
[\Rt_{\Omega}(M)]_{i,j} = \left\{
\begin{array}{cc}
M_{i,j} & (i,j) \in \Omega \\
0       & \mbox{o. w.}
\end{array}
\right.
\]
Let $Z_*$ be an optimal solution to (\ref{eqn:opt}). The recovered matrix is given by $\Mh = \Uh Z_*\Vh^{\top}$.

The following notation will be used throughout the draft. We denote by $\sigma_i, i=1, \ldots, m$ the singular values of $M$ in ranked in the descending order, and by $\u_i$ and $\v_i$ be the corresponding left and right singular vectors. Define $U = (\u_1, \ldots, \u_m)$ and $V = (\v_1, \ldots, \v_m)$. Given $r \in [m]$, partition the SVD decomposition of $M$ as
\begin{equation}
    M = U \Sigma V^{\top} =
    \bordermatrix{~ &  r & n-r \cr ~ & U_1 & U_2 \cr} %
    \begin{pmatrix} \Sigma_1 &  \\ & \Sigma_2 \end{pmatrix}
    \begin{pmatrix}V_1^{\top} \\ V_2^{\top} \end{pmatrix} \label{eqn:partition}
\end{equation}
Let $\ut_i, i \in [n]$ be the $i$th column of $U^{\top}_1$ and $\vt_i, i \in [m]$ be the $i$th column of $V^{\top}_1$. Define the incoherence measure for $U_1$ and $V_1$ as
\[
\mu(r) = \max\left(\max\limits_{i \in [n]} \frac{n}{r}|\ut_i|^2, \max\limits_{i \in [m]} \frac{m}{r}|\vt_i|^2 \right)
\]
Similarly, we define the incoherence measure for matrices $\Uh$ and $\Vh$. Let $\uh'_i, i \in [n]$ be the $i$th column of $\Uh^{\top}$ and $\vh'_i, i \in [m]$ be the $i$th column of $\Vh^{\top}$. Define the incoherence measure for $\Uh$ and $\Vh$ as
\[
\muh = \max\left(\max\limits_{i \in [n]} \frac{n}{r}|\uh'_i|^2, \max\limits_{i \in [m]} \frac{m}{r}|\vh'_i|^2 \right)
\]
Define projection operators $P_{U} = UU^{\top}$, $P_V = VV^{\top}$, $P_{\Uh} = \Uh\Uh^{\top}$, and $P_{\Vh} = \Vh\Vh^{\top}$. We will use $\|\cdot\|_2$ for spectral norm of matrix, and $\|\cdot\|_F$ for the Frobenius norm of matrix.

\section{Supporting Theorems}
In this section, we present several theorems that are important to our analysis.
\begin{thm} \label{thm:1} (~\citep{halko-2011-finding}) Let $M$ be an
  $n\times m$ matrix with singular value decomposition $M = U\Sigma
  V^{\top}$, an a fixed $r > 0$. Choose a test matrix $\Omega \in
  \R^{m\times d}$ and construct sample matrix $Y = M
  \Omega$. Partition $M$ as in (\ref{eqn:partition}) and define
  $\Omega_1 = V_1^{\top} \Omega$ and $\Omega_2 =
  V_2^{\top}\Omega$. Assuming $\Omega_1$ has full row rank, the
  approximation error satisfies
\[
    \|M - P_Y(M)\|^2_2 \leq \|\Sigma_2\|^2_2 + \|\Sigma_2\Omega_2\Omega_1^{\dagger}\|_2^2
\]
where $P_Y(M)$ project column vectors in $M$ in the subspace spanned by the column vectors in $Y$.
\end{thm}

\begin{thm} \label{lemm:1} (~\citep{tropp-2011-improved})
Let $\X$ be a finite set of PSD matrices with dimension $k$, and suppose that
\[
    \max_{X \in \X} \lambda_1(X) \leq B.
\]
Sample $\{X_1, \ldots, X_{\ell}\}$ uniformly at random from $\X$ without replacement. Compute
\[
    \mu_{\max} = \ell \lambda_{\max}(\E[X_1]), \quad     \mu_{\min} = \ell \lambda_{\min}(\E[X_1])
\]
Then
\begin{eqnarray*}
& & \Pr\left\{\lambda_{\max}\left(\sum_{i=1}^{\ell} X_i\right) \geq (1 + \delta) \mu_{\max} \right\} \leq k \left[ \frac{e^{\delta}}{(1 + \delta)^{1 + \delta}}\right]^{\mu_{\max}/B} \\
& & \Pr\left\{\lambda_{\min}\left(\sum_{i=1}^{\ell} X_i\right) \leq (1
  - \delta) \mu_{\min} \right\} \leq k \left[ \frac{e^{-\delta}}{(1 - \delta)^{1 - \delta}}\right]^{\mu_{\min}/B}
\end{eqnarray*}
\end{thm}

\begin{thm} \label{thm:perturbation}
Let $A = S^{\top}H S$ and $\tilde{A} = S^{\top}\tilde{H}S$ be two symmetric matrices of size $n\times n$. Let $\lambda_i, i \in [n]$ and $\tilde{\lambda}_i, i \in [n]$ be the eigenvalues of $A$ and $\tilde{A}$, respectively, ranked in descending order. Let $U_1, \tilde{U}_1 \in \R^{n\times r}$ include the first $r$ eigenvectors of $A$ and $\tilde{A}$, respectively. Let $\|\cdot\|$ be any invariant norm. Define
\begin{eqnarray*}
\Delta_{\lambda} & = & \min\left(\sqrt{2}\left(1 - \frac{\lambda_{r+1}}{\lambda_r}\right), \frac{1}{\sqrt{2}}\right) \\
\Delta_H         & = & \frac{\|H^{-1}\|_2 \|H - \tilde{H}\|}{\sqrt{1 - \|H^{-1}\|_2\|H - \tilde{H}\|_2}}
\end{eqnarray*}
If $\Delta_{\lambda} \geq \Delta_H/2$, we have
\[
\|\sin\Theta(U_1, \tilde{U}_1)\| \leq \frac{\Delta_H}{\Delta_{\lambda} - \Delta_H/2}\left(1 + \frac{\Delta_H\Delta_{\lambda}}{16} \right)
\]
\end{thm}
Since the above theorem follows directly from Theorem 4.4 and discussion in Section 5 from~\citep{li-1999-relative}, we skip its proof.

\section{Recovering a Low Rank Matrix}

In this section, we discuss the recovery result when the rank of $M$ is no more than $r$. We will first provide the key results for our analysis, and then present detailed proof for the key theorems.

\subsection{Main Result}

Our analysis is divided into two steps. We will first show that $|M - P_{\Uh} M P_{\Vh}|_2^2$ is small, and then bound the strongly convexity of the objective function in (\ref{eqn:opt}). The following theorem shows that the difference between $M$ and $\Mh$, measured in spectral norm, is well bounded if $|M - P_{\Uh} M P_{\Vh}|_2^2$ is small and the objective function in (\ref{eqn:opt}) is strongly convex.

\begin{thm} \label{thm:combine}
Assume (i) $\|M - P_{\Uh} M P_{\Vh}\|_2^2 \leq \Delta$, and (ii) the strongly convexity of the objective function is no less than $|\Omega|\gamma$. Then
\[
\|M - \Mh\|^2_2 \leq 2\left(\Delta + \frac{\Delta}{\gamma}\right)
\]
\end{thm}

To utilize Theorem~\ref{thm:combine}, we need to bound $\Delta$ and $\gamma$, respectively, which are given in the following two theorems.

\begin{thm} \label{thm:Delta}
With a probability $1 - 2e^{-t}$, we have,
\[
\Delta := \|M - P_{\Uh} MP_{\Vh}\|^2_2 \leq 4\sigma^2_{r+1}\left(1 + \frac{m + n}{d}\right)
\]
if $d \geq 7\mu(r) r (t+\log r)$.
\end{thm}
\begin{proof}
Our analysis is based on the following theorem.
\begin{thm} \label{thm:2}
With a probability $1 - 2e^{-t}$, we have,
\[
\|M - MP_{\Vh}\|^2_2 \leq \sigma^2_{r+1}\left(1 + 2\frac{m}{d}\right), \quad |M - P_{\Uh} M\|_2 \leq \sigma^2_{r+1}\left(1 + 2\frac{n}{d}\right)
\]
provided that $d \geq 7\mu(r) r(t+\log r)$.
\end{thm}
Using Theorem~\ref{thm:2}, we have, if $d \geq 7\mu(r) r(t+\log r)$, with a probability $1 - 2e^{-t}$
\[
\|M - P_{\Uh} M P_{\Vh}\|^2_2 \leq 2\|M - M P_{\Vh}\|^2_2 + 2\|(M - P_{\Uh}M)P_{\Vh}\|^2_2 \leq 4\sigma_{r+1}^2\left(1 + \frac{n + m}{d}\right)
\]
\end{proof}

\begin{thm} \label{thm:gamma}
With a probability $1 - e^{-t}$, we have that the strongly convexity for the objective function in (\ref{eqn:opt}) is bounded from below by $|\Omega|/2$, provided that
\[
|\Omega| \geq 7\muh^2 r^2 (t+2\log r)
\]
\end{thm}

The following lemma allows us to replace $\muh$ in Theorem~\ref{thm:gamma} with $\mu$.
\begin{thm} \label{thm:mu-1}
Assume $\mbox{rank}(M) \leq r$. Then, with a probability $1 -
2e^{-t}$, we have $\muh = \mu(r)$, provided $d \geq 7\mu(r) r(t+\log
r)$.
\end{thm}
\begin{proof}
When $\mbox{rank}(M) \leq r$, according to Theorem~\ref{thm:Delta},
with a probability $1 - 2e^{-t}$, we have $M = P_{\Uh} M P_{\Vh}$,
provided that $d \geq 7\mu(r) r (t+\log r)$. Hence $P_{U_1} = P_{\Uh}$ and $P_{V_1} =  P_{\Vh}$, which directly implies that $\mu = \muh$.
\end{proof}

The following theorem follows directly follows from Theorem~\ref{thm:Delta}, \ref{thm:gamma}, \ref{thm:mu-1}, and \ref{thm:combine}.
\begin{thm} \label{thm:low-rank} Assume $\mbox{rank}(M) \leq r$, $d
  \geq 7\mu(r) r (t+\log r)$, and $|\Omega| \geq 7\mu^2(r) r^2 (t+2\log
  r)$. Then, with a probability $1 - 3e^{-t}$, we have $M = \Mh$.
\end{thm}

\paragraph{Remark} The result from Theorem~\ref{thm:low-rank} shows
that, with a probability $1 - \delta$, a low rank matrix $M$ can be
perfectly recovered from $O((rn + r^2)\log(r/\delta))$ number of
observations from matrix $M$. This result significantly improves the
result from~\citep{krishnamurthy-20130-low}, where $O(r^2 n
\log(1/\delta))$ number of observations are needed for perfect
recovery. We should note that unlike~\citep{krishnamurthy-20130-low}
where a small incoherence measure is assumed only for column vectors
in matrix $M$, we assume a small incoherence measure for both row and
column vectors in $M$. It is this assumption that allows us to sample
both rows and columns of $M$, leading to the improvement in the sample
complexity.

\subsection{Detailed Proofs}

\subsubsection{Proof of Theorem~\ref{thm:combine}}
Set $Z = \Uh^{\top}M \Vh$. Since $\|M - P_{\Uh}M P_{\Vh}\|_2^2 \leq \Delta$, we have
\[
\|M - \Uh Z \Vh^{\top}\|_2^2 \leq \Delta,
\]
implying that
\[
|M_{i,j} - [\Uh Z \Vh^{\top}]_{i,j}|^2 \leq \Delta, \; \forall i \in [n], j \in [m]
\]
Hence, we have
\[
|\Rt_{\Omega}(M) - \Rt_{\Omega}(\Uh Z \Vh^{\top})|_F^2 \leq |\Omega|\Delta
\]
Let $Z_*$ be the optimal solution to (\ref{eqn:opt}). Using the strongly convexity of (\ref{eqn:opt}), we have
\[
\frac{\gamma}{2}|\Omega|\|Z - Z_*\|_F^2 \leq \frac{1}{2}|\Omega|\Delta,
\]
i.e. $\|Z - Z_*\|_F^2 \leq \Delta/\gamma$.
We thus have
\begin{eqnarray*}
\|M - \Mh\|_2^2 & \leq & 2\|M - P_{\Uh}M P_{\Vh}\|_2^2 + 2\|P_{\Uh} M P_{\Vh} - \Uh Z_* \Vh^{\top}\|_2^2 \\
& \leq & 2\|M - P_{\Uh}M P_{\Vh}\|_2^2 + 2\|Z - Z_* \|_2^2 \leq 2\left(\Delta + \frac{\Delta}{\gamma}\right)
\end{eqnarray*}

\subsubsection{Proof of Theorem~\ref{thm:2}}
Let $i_1, \ldots, i_d$ are the $d$ selected columns. Define $\Omega = (\e_{i_1}, \ldots, \e_{i_d}) \in R^{m\times d}$, where $\e_i$ is the $i$th canonical basis. To utilize Theorem~\ref{thm:1}, we need to bound the minimum eigenvalue of $\Omega_1\Omega_1^{\top}$. We have
\[
\Omega_1\Omega_1^{\top} = V_1^{\top}\Omega\Omega^{\top} V_1
\]
Let $\vt_i, i \in [d]$ be the $i$th row vector of $V_1$. We have
\[
\Omega_1\Omega_1^{\top} = \sum_{j=1}^d \vt_{i_j}\vt_{i_j}^{\top}
\]
It is straightforward to show that
\[
\E\left[\Omega_1\Omega_1^{\top} \right] = \frac{d}{m}I_r
\]
To bound the minimum eigenvalue of $\Omega_1\Omega_1^{\top}$, we will use Theorem~\ref{lemm:1}.
To this end, we have
\[
B = \max\limits_{1 \leq i \leq m} |\vt_i|^2 \leq \mu(r) \frac{r}{m}
\]
Thus, we have
\[
\Pr\left\{\lambda_{\min}(\Omega_1\Omega_1^{\top}) \leq (1 -
  \delta)\frac{d}{m}\right\}\leq r\cdot \exp\left(- \frac{d}{\mu(r) r}\left[\delta + (1 - \delta)\ln(1 - \delta)\right]\right)
\]
By setting $\delta = 1/2$, we have, with a probability $1-r e^{-d/[7\mu(r) r]}$
\[
\lambda_{\min}(\Omega_1\Omega_1^{\top}) \geq \frac{d}{2m}
\]
Under the assumption that
\[
\lambda_{\min}(\Omega_1\Omega_1^{\top}) \geq \frac{d}{2m},
\]
using Theorem~\ref{thm:1}, we have
\[
\|A - AP_{\Vh}\|_2^2 \leq \sigma_{r+1}^2 + \left\|\Sigma_2\Omega_2\Omega_1^{\dagger}\right\|_2^2 \leq \sigma_{r+1}^2 + \frac{2m}{d}\left\|\Sigma_2\Omega_2\right\|_2^2 \leq \sigma_{r+1}^2\left(1 + \frac{2 m}{d}\|\Omega_2\|_2^2\right)
\]
We complete the proof using the fact that $\|\Omega_2\|_2 \leq 1$.

\subsubsection{Proof of Theorem~\ref{thm:gamma}}
We rewrite the objective function as
\[
\|\Rt_{\Omega}(M) - \Rt_{\Omega}(\Uh Z\Vh)\|_F^2 = \left\|\Rt_{\Omega}(M) - \sum_{i=1}^r \sum_{j=1}^r Z_{s,t} \Rt_{\Omega}(\uh_i \vh_j^{\top})\right\|_F^2
\]
Define matrix $K = (\k_1, \ldots, \k_{r^2}) \in \R^{nm\times r^2}$, where $\k_{(i,j)} = \vec\left(\Rt_{\Omega}(\uh_i\vh_j^{\top})\right)$. Our goal is to bound the minimum eigenvalue of $K^{\top}K$. To Theorem~\ref{lemm:1}, we bound
\[
B = \max\limits_{i,j} |\k_{(i,j)}|^2 \leq \frac{\muh^2 r^2}{mn}
\]
and
\[
\lambda_{\min}\left(\E[K^{\top}K]\right) = \frac{|\Omega|}{mn} \lambda_{\min}\left(\Uh \otimes \Vh\right) = \frac{|\Omega|}{mn}
\]
where $\otimes$ is Kronecker product. Thus, according to Theorem~\ref{lemm:1}, with a probability $1 - e^{-t}$, we have
\[
\lambda_{\min}(K^{\top}K) \geq \frac{|\Omega|}{2mn}
\]
provided that
\[
|\Omega| \geq 7 \muh^2r^2 (t+2\log r)
\]

\section{Recovering the Low Rank Approximation of a Full Rank Matrix}

In this section, we consider a general case when $M$ is of full rank but with skewed eigenvalue distribution. To capture the skewed eigenvalue distribution, we use the concept of numerical rank $r(M, \lambda)$ with respect to non-negative constant $\lambda > 0$, which is defined as follows~\citep{hansen-1987-rank}
\[
r(M, \lambda) = \sum_{i=1}^m \frac{\sigma_i^2}{\sigma_i^2 + mn\lambda}
\]
Define
\[
H_A = \lambda I + \frac{1}{mn}MM^{\top}, \quad \Hh_A = \lambda I + \frac{1}{dn}AA^{\top}
\]
and
\[
H_B = \lambda I + \frac{1}{mn}M^{\top}M, \quad \Hh_B = \lambda I + \frac{1}{dm}BB^{\top}
\]

Next, we generalize the definition of incoherence measure to numerical low rank. Define $S = \Sigma^2 + mn\lambda I$, where $\Sigma = \mbox{diag}(\sigma_1, \ldots, \sigma_m)$, and incoherence measure $\mu$ with respect to a non-negative constant $\lambda > 0$ as
\[
\mu(\lambda) = \max\left(\max\limits_{1 \leq i \leq n} \frac{n}{r(M, \lambda)}|V_{i,*}\Sigma S^{-1/2}|^2, \max\limits_{1 \leq i \leq n} \frac{m}{r(M, \lambda)}|U_{i,*}\Sigma S^{-1/2}|^2 \right)
\]
It is easy to verify that $\mu(\lambda) \geq 1$. Note that when the rank of matrix $M$ is $r$, we have $r(M, 0) = r$ and $\mu(0) = \mu(r)$.

In order to utilize the theorems presented in Section 4 to bound $\|M - \Mh\|_2$, the key is to bound $\mu(r)$ and $\muh$ by $\mu(\lambda)$. The following theorem allows us to bound $\mu(r)$ by $\mu(\lambda)$.
\begin{lemma}
Assume
\[
\frac{\sigma_r^2}{(\sigma_r^2 + mn \lambda)r(M, \lambda)} \geq \frac{a}{r}
\]
for some positive $a > 0$. We have $\mu(\lambda) \geq a \mu(r)$. More specifically, if we choose $\lambda = \sigma_r^2/mn$, we have
\[
\mu(r) \leq \frac{2r(M, \lambda)}{r} \mu(\lambda)
\]
\end{lemma}
Using the above lemma, we have a modified version for Theorem~\ref{thm:Delta}
\begin{thm} \label{thm:Delta-1} Set $\lambda = \sigma^2_r/mn$ for a
  fixed $r$. With a probability $1 - 2e^{-t}$, we have,
\[
\Delta := \|M - P_{\Uh} MP_{\Vh}\|^2_2 \leq 4\sigma^2_{r+1}\left(1 +
  \frac{m + n}{d}\right)
\]
if $d \geq 14 \mu(\lambda) r(M, \lambda) (t+\log r)$.
\end{thm}
We note that Theorem~\ref{thm:Delta-1} is almost identical to Theorem~\ref{thm:Delta} except that $\mu(r)r$ is replaced with $\mu(\lambda) r(M, \lambda)$.

Next we will bound $\muh(r)$ by $\mu(\lambda)$. To this end, we need the following theorem.
\begin{thm} \label{thm:3}
With a probability $1 - 4e^{-t}$, for any $k \in [n]$, we have
\begin{eqnarray*}
& 1 - \delta  \leq \lambda_{k}(H_A^{-1/2}\Hh_A H_A^{-1/2}) \leq 1 + \delta & \\
& 1 - \delta \leq \lambda_{k}(H_B^{-1/2}\Hh_B H_B^{-1/2}) \leq 1 + \delta &
\end{eqnarray*}
if
\[
d \geq \frac{4}{\delta^2}(\mu(\lambda) r(M, \lambda) + 1)(t + \log n)
\]
\end{thm}

\begin{thm} \label{thm:muh-1}
Assume that $d \geq 16(\mu(\lambda)r(M, \lambda) + 1) (t + \log n)$, and $\sigma_r \geq \sqrt{2}\sigma_{r+1}$. Set $\lambda = \sigma_r^2/mn$. With a probability $1 - 4e^{-t}$, we have
\[
\muh(r) \leq \frac{2r(M, \lambda)}{r}\mu(\lambda) + \frac{18 n \delta^2}{r}
\]
where
\[
\delta^2 = \frac{4}{d} (\mu(\lambda)r(M, \lambda) + 1)(t + \log n)
\]
\end{thm}

Using Theorem~\ref{thm:muh-1}, we have the following version of Theorem~\ref{thm:gamma}
\begin{thm} \label{thm:gamma-1}
Assume $d \geq 16(\mu(\lambda) r(M, \lambda) + 1)(t + \log n)$ and $\sigma_r \geq \sqrt{2} \sigma_{r+1}$. Set $\lambda = \sigma_r^2/mn$. With a probability $1 - 5e^{-t}$, we have that the strongly convexity for the objective function in (\ref{eqn:opt}) is bounded from below by $|\Omega|/2$, provided that
\[
|\Omega| \geq 7\left(2\mu(\lambda)r(M, \lambda) +
  72\frac{n}{d}(\mu(\lambda) r(M, \lambda) + 1)(t + \log
  n)\right)^2(t+2\log r)
\]
\end{thm}

Combining the above results, we have the final theorem for the recovering of $M$ when its numerical rank $r(M, \lambda)$ is small.

\begin{thm} \label{thm:low-rank-2}
Assume $d \geq 16(\mu(\lambda) r(M, \lambda) + 1)(t + \log n)$ and $\sigma_r \geq \sqrt{2} \sigma_{r+1}$. Set $\lambda = \sigma_r^2/mn$. We have, with a probability $1 - 7e^{-t}$
\[
\|M - \Mh\|_2^2 \leq 24\sigma_{r+1}^2\left( 1 + \frac{(m+n)}{d}\right)
\]
if
\[
|\Omega| \geq 7\left(2\mu(\lambda)r(M, \lambda) +
  72\frac{n}{d}(\mu(\lambda) r(M, \lambda) + 1)(t + \log
  n)\right)^2(t+2\log r)
\]
\end{thm}

\paragraph{Remark} The total number of observed entries are $\tilde O(dn +
n^2/d^2)$. It is minimized when $d = n^{1/3}$, leading to $\tilde O(n^{4/3})$
for the number of observed entries and $\tilde O(\sigma_{r+1}n^{1/3})$ for
recovery error.

\subsection{Detailed Proof}
\subsubsection{Proof of Theorem~\ref{thm:3}}
It is sufficient to show the result for $H_A^{-1/2}\Hh_A H_A^{-1/2}$. Define
\[
\X = \left\{M_i = H_A^{-1/2}\left(\frac{1}{n}\a_i\a_i^{\top} + \lambda I\right)H_A^{-1/2}, i=1, \ldots, m \right\}
\]
We have
\begin{eqnarray*}
M_i & = & US^{-1/2}U^{\top}(m U\Sigma V_{i, *}^{\top}V_{i,*}\Sigma U^{\top} + mn\lambda I)US^{-1/2}U^{\top} \\
& = & U\left(m S^{-1/2}\Sigma V_{i, *}^{\top}V_{i,*}\Sigma S^{-1/2} + mn\lambda S^{-1}\right)U^{\top}
\end{eqnarray*}
Using the definition of $\mu(\lambda)$, we have $\lambda_{\max}(M_i) \leq \mu r(M, \lambda) + 1$. Since
\[
B = d \lambda_{\max}(\E[M_i]) = d
\]
we have
\[
\Pr\left\{\lambda_{\max}\left(H_A^{-1/2}\Hh_A H_A^{-1/2} \right) \geq 1 + \delta \right\} \leq n\exp\left( -\frac{d}{\mu(\lambda) r(M,\lambda) + 1}\left[(1 + \delta)\log(1 + \delta) - \delta\right]\right)
\]
Using the fact that
\[
(1 + \delta) \log(1 + \delta) \geq \delta + \frac{1}{4}\delta^2, \forall \delta \in [0, 1],
\]
we have
\[
\Pr\left\{\lambda_{\max}\left(H_A^{-1/2}\Hh_A H_A^{-1/2} \right) \geq 1 + \delta \right\} \leq n\exp\left( -\frac{d \delta^2}{4(\mu r(M,\lambda) + 1)}\right)
\]
The upper bound is obtained by setting $d = 4(\mu(\lambda) r(M, \lambda) + 1) (\log n + t)/\delta^2$. Similarly, for the lower bound, we have
\[
\Pr\left\{\lambda_{\min}\left(H_A^{-1/2}\Hh_A H_A^{-1/2}\right) \leq 1 - \delta \right\} \leq n\exp\left( - \frac{n}{\mu(\lambda) r(M,\lambda) + 1}\left[(1 - \delta)\log(1 - \delta) + \delta\right]\right)
\]
Using the fact that
\[
(1 - \delta)\log(1 - \delta) \geq - \delta + \frac{\delta^2}{2}
\]
We have the lower bound by setting $m = 2(\mu(\lambda) r(M, \lambda) + 1) (\log n + t)/\delta^2$.

\subsubsection{Proof of Theorem~\ref{thm:muh-1}}
To utilize Theorem~\ref{thm:perturbation}, we rewrite $H_A$ and $\Hh_A$, defined in Theorem~\ref{thm:3}, as
\[
H_A = H_A^{1/2} I H^{1/2}_A, \quad \Hh_A = H_A^{1/2}DH_A^{1/2}
\]
where $D = H_A^{-1/2} \Hh_A H_A^{-1/2}$. According to Theorem~\ref{thm:3}, with a probability $1 - 2e^{-t}$, we have $\|D - I\|_2 \leq \delta$, provided that
\[
d \geq \frac{4}{\delta^2}(\mu(\lambda) r(M, \lambda) + 1)(t + \log n)
\]
We then compute $\Delta_{\lambda}$ and $\Delta_{H}$ defined in Theorem~\ref{thm:perturbation}. Using the fact $d \geq 16(\mu(\lambda) r(M,\lambda) + 1)(t + \log n)$ and Theorem~\ref{thm:3}, we have, with a probability $1 - e^{-t}$, $\delta \leq 1/2$. Hence
\begin{eqnarray*}
\Delta_H \leq \frac{\|D - I\|_2}{\sqrt{1 - \|D - I \|_2}} = \frac{\delta}{\sqrt{1 - \delta}} \leq \sqrt{2}\delta
\end{eqnarray*}
Using the assumption that $\sigma_r/\sigma_{r+1} \geq \sqrt{2}$, we have $\delta \leq 1/2 \leq 1 - \sigma^2_{r+1}/\sigma^2_r$ and therefore $\Delta_{\lambda} = 1/\sqrt{2}$. As a result, according to Theorem~\ref{thm:perturbation}, we have
\[
\|\sin\Theta(U_1, \Uh)\|_2 \leq 3\sqrt{2}\delta
\]
Similarly, we have,
\[
\|\sin\Theta(V_1, \Vh)\|_2 \leq 3\sqrt{2}\delta
\]
Thus, with a probability $1 - 4e^{-t}$, we have
\[
\muh(r) \leq \mu(\lambda) + \frac{n}{r}\|\sin\Theta(V_1, \Vh)\|^2_2  \leq \mu(\lambda) + \frac{18n \delta^2}{r}
\]

\bibliography{cur-matrix-completion}

\end{document}